\documentclass{ecai2014}
\usepackage{times}
\usepackage{booktabs}
\usepackage{multirow}
\usepackage{url}
\usepackage{graphicx}
\usepackage{amsmath, amssymb}
\usepackage{mathptmx}
\usepackage{cite}
\usepackage{xspace}
\usepackage{color}

\newenvironment{proof}{\noindent{\bf Proof.}\hspace*{1em}}{\literalqed\smallskip}
\def\literalqed{{\ \nolinebreak\hfill\mbox{\qedblob\quad}}}

\newcommand\qedblob{\mbox{$\square$}}

\renewcommand{\theenumii}{\alph{enumii}}
\renewcommand{\theenumiii}{\roman{enumiii}}

\newtheorem{mycorollary}{Corollary}

\newtheorem{mytheorem}{Theorem}

{\bf}{\it}

\newcommand{\STV} {\mbox{\sc STV}}

\newcommand{\Hyb} {\mbox{\sc Hyb}}

\newcommand{\OMIT}[1]{} 

\title{How Hard Is It to Control an Election by Breaking Ties?}

\author{Nicholas Mattei\institute{NICTA and UNSW, Sydney, Australia.}
\and Nina Narodytska\institute{University of Toronto and UNSW, Toronto, ON, Canada.}
\and Toby Walsh\institute{NICTA and UNSW, Sydney, Australia} 
}

\begin{document}

\maketitle

\begin{abstract}
We study the computational complexity of controlling the result of 
an election by breaking ties strategically.
This problem is equivalent to the problem of deciding the winner of an election under
parallel universes tie-breaking.
When the chair of the election is only asked to break ties
to choose between one of the co-winners, the problem is trivially easy. 
However, in
multi-round elections, we prove that it can be NP-hard for the chair
to compute how to break ties to ensure a given result.  Additionally, 
we show that the form of the tie-breaking function can increase the opportunities for control. 
Indeed, we prove that it can be NP-hard to control an election by 
breaking ties even with a two-stage voting rule.
\end{abstract}

\section{INTRODUCTION}

Voting is a general mechanism to combine individual
orderings into a group preference (e.g. preferences of agents over
different plans, or rankings of web pages by
different search engines). One concern that the individual agents may have
is that the chair may manipulate the result.
For example, the chair might introduce a
spoiler candidate or delete some
votes. Bartholdi, Tovey and Trick
\cite{bttmcm92} explored
an interesting barrier to such manipulation; perhaps
it is computationally too difficult for the
chair to work out how to
perform such control? They proved that many types of
control problems are NP-hard even for simple
voting rules like plurality.

Interestingly, one type of control not considered
by Bartholdi, Tovey and Trick is control
by choosing how ties are broken. This is
surprising since the chair
is actually the person who breaks ties
in many elections.
For example, the Speaker in many parliaments has
the casting vote in case of a tied vote.
Another reason to consider such control
is that in many elections the tie breaking rule
is unspecified or is left ambiguous.
The chair therefore has an opportunity
to influence the outcome by selecting a beneficial (to him) rule.

This control problem also avoids one of
the criticisms raised against the analysis of some of the other forms of
control. In particular, many complexity results about
control 
suppose that the chair
has complete knowledge of the votes. This might
be considered unreasonable.
For example, how do we know how voters will rank a new
spoiler candidate till their candidature has been announced?
When studying control by
breaking ties, it is natural
to suppose the chair knows how the votes are cast
when asked to break a tie.

Control by tie-breaking is equivalent to
the problem of determining if a chosen alternative can win
under \textit{some} tie-breaking rule, an idea known as parallel universes tie-breaking (PUT)
\cite{conitzer2009preference}.  As PUT does not instantiate a particular tie-breaking rule, 
but rather the set of all tie-breaking rules, 
there is no longer a dependency on the names of the individual
candidates.  This property, known as neutrality, can be restored under PUT whereas
tie-breaking off a lexicographical ordering does not allow for neutrality in the final
resolute voting rule \cite{tideman1987independence}. 
Deciding if a candidate is the winner of such
a neutral rule with ranked pairs voting has recently been shown to
be NP-complete \cite{BrillF12}; it follows that control by tie-breaking is NP-complete.
Winner determination under is PUT 
is also closely related to determining if a given alternative
has a chance to win in the presence of certain types of uncertain
information \cite{aziz2012possible,xia2011determining}

Tie-breaking has played an important role
in some of the earliest
literature on computational social
choice. For example,
Bartholdi, Tovey and Trick \cite{bartholditoveytrick}
proved that a single agent can manipulate
a Copeland election in
polynomial time when ties are broken
in favour of the manipulators, but
manipulation becomes NP-hard
when the tie-breaking rule used in chess competitions is employed.
With Copeland voting,
Faliszewski, Hemaspaandra, and Schnoor \cite{copeland}
proved that the choice of how ties are scored
can change the computational complexity of
computing a manipulation from polynomial
to NP-hard.
More recently work by
Obraztsova, Elkind and Hazon \cite{oehaamas11} and
Aziz et al. \cite{aziz-rv} considered
the impact of different randomized tie-breaking schemes on
the computational complexity of computing
a manipulation. They proved, for example, that
all scoring rules are polynomial to manipulate for some tie-breaking rules but not others;
additionally rules like maximin, STV and ranked pairs are NP-hard.

In this paper, we study the computational
complexity of control by breaking ties.  While ties in a real elections may not 
be that common, they have
been observed. For instance,  US Vice Presidents
have had to cast tie-breaking votes in 244 Senate votes.
Indeed John Adams, the first Vice President, cast 29 such votes.  Often elections that are not closely contested cannot be manipulated \cite{xcec08b}
and therefore, tied elections being the most closely contested of all, represent an interesting
edge case that has not been greatly investigated.
We show that when the chair only
breaks ties to choose between co-winners,
as is the case in many single round
rules, control by tie-breaking is polynomial.
On the other hand, for many multi-round
rules like Coombs, Cup, and STV,
the chair may have to break multiple ties,
and the control by tie-breaking problem is
NP-complete. Even with two-stage rules,
where the chair may have to break ties only twice,
the control by tie-breaking problem can
be NP-complete.

\section{FORMAL BACKGROUND}\label{sec:background}
An election is defined by a set of \emph{candidates} $C$ with $|C|=m$, 
a {\em profile} $P$ which is a set of $n$ strict linear orders (votes) over $C$, 
and a {\em voting correspondence} $R$.
Let $R$ be a function $R: P \rightarrow W$ mapping a profile
onto a set of {\em co-winners} where $W \subseteq C$.
If $|W|=1$ then we have a \textit{voting rule}, otherwise
we may require a tie-breaking rule $T$ that will
return a unique winner (single element) from
$W$.  Let $N(i,j)$ be
the number of voters preferring $i$ to $j$.
We consider the following voting rules in this study \cite{arrow:b:handbook}.

\begin{description} \itemsep=0pt
\item[Scoring rules:]
$(w_1,\ldots,w_m)$ is a vector
of weights where the $i$th candidate in a vote
scores $w_i$, and the co-winners are the candidates with
highest total score. 
{\em Plurality} has
$w_1=1$, and $w_i=0$ for $i>1$;
{\em veto} has
$w_i=1$ for $i<m$,
and $w_m=0$, ;
$k$-{\em approval} has
$w_i=1$ for $i\leq k$, and $w_i=0$ for $i>k$;
{\em Borda} has $w_i=m-i$.
\item[Plurality with runoff:]
If one candidate has a majority,
then she wins. Otherwise we eliminate
all but the two candidates with the most
votes and apply the plurality rule.
\item[Black's Rule:]
If one candidate is the Condorcet winner, a
candidate preferred by a majority of voters to all others, then
she wins. Otherwise, we apply the Borda rule.
\item[Bucklin:] The Bucklin 
score of a candidate is his $k$-approval score with $k$ set to be
the smallest value such that the $k$-approval score 
of at least one candidate exceeds $\lfloor n/2 \rfloor$.
The co-winners are the candidates with the largest Bucklin score.  The \textit{simplified Bucklin} procedure
is the same except that all candidates with score exceeding $\lfloor n/2 \rfloor$ are co-winners.
\item[Fallback:]
This is a combination of Bucklin and approval
voting. Voters approve and rank a subset of the candidates. 
If there is a $k$ such that the $k$-approval
score of at least one candidate, the sum of the approvals
appearing in the first $k$ places of each voter's ranked order,
 exceeds $\lfloor n/2 \rfloor$
then the co-winners are the set of candidates exceeding
this threshold. If there is no such $k$ (as no candidate
receives enough approvals), the winner is the
approval winner.
\item[Single Transferable Vote (STV):]
This rule requires up
to $m-1$ rounds. In each round,
the candidate with the least number of voters ranking
him first is eliminated until one of the remaining candidates
has a majority.
\item[Nanson and Baldwin:]
These are elimination versions of Borda voting. In each
round of Nanson, we eliminate all candidates with less than
the average Borda score. In each round of Baldwin,
we eliminate the candidate with the lowest Borda score.
\item[Coombs:]
This is the elimination version of veto voting. In each
round, we eliminate the candidate with the lowest veto score
until we have one candidate with a plurality score
of $n/2$ or greater. In the simplified
version of Coombs, we eliminate the candidate with the
lowest veto score until one candidate remains.
\item[Cup:]
Given a schedule $S$ and a labeling $L$,
we run a knockout tournament. Candidates
are compared pairwise with the winner in each moving to the next round.
The overall winner is the candidate to win the final matchup round.
\item[Copeland$^\alpha$:]
The candidates with the highest Copeland$^\alpha$ score win.
The Copeland$^\alpha$ score of candidate $i$
is $\sum_{j \neq i} (N(i,j)>\frac{n}{2}) + \alpha(\sum_{j \neq i}  (N(i,j) = \frac{n}{2}))$.
In the second order Copeland rule, 
if there is a tie, the winner is the candidate whose defeated
competitors have the largest sum of Copeland scores.
\item[Ranked pairs:]
We consider all pairs of candidates in order of the pairwise margin of
victory, from greatest to least.  For each considered pair, we
construct an ordering which ranks these candidates
unless it creates a cycle.
The winner is the candidate at the top of the constructed ordering.
For a non-neutral variant, when there are two or more pairwise relations with the same
amount of support, we resolve ties according to an outside ordering.
For a neutral variant,
the co-winners are any candidate who can be made top element under
some tie breaking order \cite{BrillF12}.
\item[Maximin:]
The Maximin score of a candidate is the number of votes received in
his worst pairwise election.  The co-winners are
the candidates with the largest such score.
\item[Schulze Method:] The Schulze ranking \cite{schulze2011new} of 
candidates is computed from the pairwise majority graph
where an edge between candidates $i$ and $j$ is weighted
by $N(i,j)$.  A beatpath score is 
computed for all candidates, which is the maximum weight 
path to all other candidates.  The winning set is the set of 
candidates with highest beathpath scores.
\item[Kemeny-Young:]
The Kemeny-Young rule selects the ranking
with maximal Kemeny score \cite{kemeny1959mathematics,young1995optimal}.
The Kemeny score of a ranking is measured by summing, for each candidate
pair $i, j \in C$, $N(i,j)$.  Finding 
the ranking(s) with highest Kemeny score(s) is computational hard to compute when $m \geq 4$ 
\cite{bartholdi1989voting,hemaspaandra2005complexity}.
\end{description}

A \emph{tie-breaking} rule $T$ for an election is a single
valued choice function that, for any subset $W \subseteq C$, $W \neq \emptyset$, and profile $P$,
$T(P, W)$ returns a single element $c \in W$ \cite{oehaamas11}.
Commonly, $T$ is a strict linear order over $C$ that is provided aprori (e.g. by age or alphabetically).
However, this definition allows us to represent functions that are 
not necessarily a linear order over the candidate set.
This includes functions that are not
transitive for all the candidates, and,   
while this can be seen as an undesirable
propriety in elections, it is a common element in sports 
competitions (e.g., NCAA Football) where aspects like goal
differential and total points scored are used as non-transitive tie-breaking functions.
We consider the following
decision problem.

\smallskip
\noindent \textbf{Name:} \textsc{Control by Tie-Breaking} \\
\noindent \textbf{Question:} Given profile $P$ and preferred candidate $p \in C$,
is there a tie-breaking rule $T$ such that $p$ can be made the unique winner
of the election under voting rule $R$?

\smallskip
A voting rule is {\em vulnerable} to such
control if this problem is polynomial,
and {\em resistant} if it is NP-hard.
All voting rules that require tie-breaking at some point, which includes all voting rules
presented in this section,
are \textit{susceptible} to this form of control.

In the {\em manipulation problem} \cite{bttmcm92},
we wish to decide if we can cast
one additional vote to make $p$
win. All our results here apply to
the variants of the manipulation problem in which we
break ties in favour of or
against the manipulator. Finally,
in the {\em manipulation problem with random
tie-breaking} \cite{aziz-rv,oehaamas11}, we are
also given a probability
$t$ and we wish to decide if we can cast
one additional vote to make $p$ the
winner with probability at least $t$
supposing ties are broken uniformly at
random between candidates.

\section{RELATIONSHIP TO MANIPULATION}
We start by considering how control by breaking
ties is related to other manipulation problems.
A little surprisingly, the 
complexity
of control by breaking ties is not related
to that of the manipulation problem
with random tie-breaking or the standard manipulation problem
when ties are broken in a fixed order.

\begin{mytheorem}
There exists a voting correspondence such that
the control by tie-breaking problem is polynomial
but the manipulation problem with random tie-breaking
is NP-complete (and vice versa).
\end{mytheorem}

\begin{proof}
In Corollary \ref{cor:copeland}, we prove that the control by
tie-breaking problem for Copeland is polynomial.
On the other hand, Obraztsova and Elkind \cite{oeijcai11} prove that the manipulation
problem with random tie-breaking for Copeland is NP-complete.

Consider the voting rule that eliminates half the candidates
using the veto rule, then elects the plurality
winner. In Theorem \ref{thm:vetoplurality}, we prove that the control by
tie-breaking problem for this rule is NP-complete.
However, the manipulation
problem with random tie-breaking for this
rule is polynomial
since we can exhaustively try all
$m(m-1)/2$ votes with different candidates
in the first and the last position.
\end{proof}


\begin{mytheorem}
There exists a voting correspondence such that
the control by tie-breaking problem is polynomial
but the manipulation problem
is NP-complete (and vice versa).
\end{mytheorem}
\begin{proof}
In Theorem \ref{thm:easy}, we prove that the control by
tie-breaking problem for Nanson is polynomial.
On the other hand, the manipulation problem
for Nanson is NP-complete \cite{nwxaaai11}.

Consider again the voting rule that eliminates half the candidates
using the veto rule, then elects the plurality
winner. In Theorem \ref{thm:vetoplurality}, we prove that the control by
tie-breaking problem for this rule is NP-complete.
However, the manipulation
problem is polynomial
since we can exhaustively try all
$m(m-1)/2$ votes with different candidates
in first and last position.
\end{proof}


\section{SELECTING FROM THE WINNING SET}
We start with some very simple cases.
When tie-breaking only ever takes place
once and at the end, then the chair is choosing between the co-winners.
In such cases, control by breaking ties is
trivially polynomial. The chair
can ensure a candidate $p$
wins if and only if $p$ is amongst the
co-winners.

\begin{mytheorem}\label{thm:easy}
The control by tie-breaking problem when we select 
from among a set of co-winners once is polynomial.
\end{mytheorem}

Theorem~\ref{thm:easy} covers the majority of 
voting rules presented in Section~\ref{sec:background}.  Specifically, control by tie-breaking
is easy for: 
\begin{itemize}
\item All scoring rules, Bucklin, Black, maximin,  and Copeland$^{\alpha}$ for any $\alpha$ 
are polynomial.
\item Plurality with runoff is polynomial since only $O(m)$ candidates can enter
the runoff with the candidate we wish to win, so we can try all possibilities.
\item Fallback is polynomial, this is interesting as it holds the current record of 
resistance to 20 of the 22 methods of control \cite{fallback}.
\item Nanson's rule is a multi-round rule where manipulation is NP-complete \cite{nwxaaai11}.  However, control by tie-breaking is polynomial. Since Nanson's rule eliminates all alternatives with less than the average Borda score, the only time that it breaks ties is in the final round when multiple candidates have the maximal Borda score.
\item Schulze Method is also polynomial. In the two common implementations, edge(s) between candidate $i$ and $j$ is
either $N(i,j)$ or the margin of votes $N(i,j) - N(j,i)$.  In the former case, even if these 
numbers are tied it does not imply a tie in the outcome ordering of $i$ and $j$.
In the latter case, if there is a tie then the edge appears in the graph with $0$ weight.  
The only possible tie in the method occurs when two candidates have the same beatpath 
score.
Officially Schulze rule requires that an order is drawn at random from $P$ and ties are resolved according to this order.  If we assume the
chair can select a ballot then he just selects the ballot that is closest to his true preference.
\item For the Kemeny-Young method with $m \leq 3$, the only ties that can occur are between pairs of elements in the outcome ordering,
we need to select a resolution of these pairs such that $p$ wins.  This can be computed in polynomial time through brute force
computation as the number of possible resolutions of the pairwise ties is polynomial.  Note that the problem is trivially
hard for instances where $m \geq 4$.
\end{itemize}

To have any resistance to control by
tie-breaking, we
need more complex tie-breaking.
One place to see more complexity 
is with multi-round rules like
STV and Coombs in which
candidates are successively eliminated.
Such rules increase the number of times
ties may need to be broken.

The Copeland rule offers another interesting
control opportunity for the chair.
The chair might be in a position to
set $\alpha$, the score that a candidate
receives in the event of a tie in the tournament
graph. The choice of $\alpha$ has an
impact on the computational complexity
of computing a manipulation \cite{copeland,faliszewski2009llull}.
What happens when we hand the choice
of $\alpha$ over to the manipulator? He just needs
to find $\forall c \in C:
wins(p) + \alpha \cdot ties(p) \geq wins(c) + \alpha \cdot ties(c)$. Where $wins(p)$ is 
the number of points received for wins and $ties$ is the number of points received 
for tied competitions.  Since
$\alpha$ must be a rational number between $0 \leq \alpha \leq 1$ we can find it quickly 
with a short linear program.

\begin{mytheorem}
The control problem of setting $\alpha$ for
Copeland\textsuperscript{$\alpha$} is polynomial.
\end{mytheorem}


\section{BREAKING TIES DURING EXECUTION}
We now move to multi-round voting rules.
Bartholdi and Orlin \cite{stvhard} showed that
the manipulation problem for STV
is NP-complete; Conitzer et al.~\cite{conitzer2009preference}
showed that the winner determination under PUT for STV, and therefore control by tie-breaking, is
NP-complete.


%
%
%

\subsection{Baldwin and Coombs rules}

We next consider Baldwin and Coombs's voting
rules. These are multi-round rules that
successively eliminate candidates based on
their Borda or Veto scores, respectively.
The manipulation problem for Baldwin's rule
is NP-complete and we can modify the proof given
by Narodytska et al.~\cite{nwxaaai11}.

\begin{mytheorem}
The control by tie-breaking problem for Baldwin's rule is NP-complete.
\end{mytheorem}

\begin{proof}\emph{(Sketch)}
We modify the NP-completeness proof for Baldwin manipulation \cite{nwxaaai11}
to move the burden of finding the exact cover from the manipulator and onto the tie-breaking
rule. The chair will set the tie-breaking order
such that we select exactly a subset of sets that give us an exact cover in an instance of
\textsc{Exact Cover by 3 Sets} (\textsc{X3C}).
%
Given two sets $V = \{v_1, \dots, v_q\}$, $q = 3t$, and $S = \{S_1, \dots, S_t\}$,
where $t \geq 2$ and for all $j \leq t$, $|S_j| = 3$, and $S_j \subseteq V$ we create an instance
with the set of candidates $C = \{p,d, b\} \cup V \cup A$. Note that $p$ is the preferred candidate,
members of $A = \{a_1, \dots, a_t\}$ correspond to the 3-sets in $S$, and $m = |C| = q + t + 3$.
The construction is made up of two parts.  The first set of votes $P_1$ remains unchanged
from Narodytska et al.~\cite{nwxaaai11} and is used to control changes in the score difference between candidates
as they are eliminated.
The second set of votes $P_2$, which are the votes
that set the initial score differences between the candidates, are modified so
that the dangerous candidates are tied with $p$ (rather than having one more vote than $p$).

We will make use of the votes $W_{(u,v)} = \{(u \succ v \succ Others), (rev(Others) \succ u \succ v)\}$, where $Others$
are all the candidates in $C \setminus \{u,v\}$ in lexicographical order.  Votes of this form (1)
give $m$ points to $u$, $m-2$ points to $v$, and $m-1$ points to all other candidates; and (2)
have the property that for any set of candidates $C' \subseteq C$ and any pair of candidates
$x,y \in C \setminus C'$ if $x = v$ and $u$ is removed, the score of $v$ to increase by $1$,
if $x = u$ and $v$ is removed, the score of $u$ decreases by $1$, otherwise the scores
are unchanged.  
These votes allow us to construct a profile such that removing
candidates in a particular order creates ties in the next round.

The set of votes $P_1$ is the machinery that creates a series of ties that we must
select from and is unchanged from Narodytska et al.~\cite{nwxaaai11}.
Let $s_{base}(P_1) = m(6mt +mq + m(t+6))$.  

The votes
in $P_1$ are:
\begin{itemize}
  \item  for each $j \leq t$ and each $v_i \in S_j$ there are $2m$ copies of $W_{(v_i, a_j)}$;
  \item for each $i \leq q$, there are $m$ copies of $W_{(b,v_i)}$;
  \item $m(t+6)$ copies of $W_{(b,p)}$.
\end{itemize}

The set of votes $P_2$ set the initial score differences between the candidates.
Let $s_{base}(P_2) = m(m(7t + 5 - q) + (mt^2) + 2m(t+6))$.  

\begin{itemize}
  \item  for each $i \leq q$, there are $2m \cdot occ(i) + mt + 4m$ copies of $W_{(d,v_i)}$;
  \item for each $j \leq t$, there are $mt$ copies of $W_{(d, a_j)}$;
  \item $2m(t+6)$ copies of $W_{(d,b)}$.
\end{itemize}

This gives the candidates the following scores for the votes in $P_2$:
\begin{align*}
&s(v_i,P_2) = s_{base}(P_2) - (2m \cdot occ(i) + mt + 4m) \\
&s(a_j,P_2)= s_{base}(P_2) - (mt) \\
&s(p,P_2) = s_{base}(P_2) \\
&s(b,P_2) = s_{base}(P_2) -  2m(t+6) \\
&s(d,P_2) = s_{base}(P_2) + m(7t+5-q) + (mt^2) + 2m(t+6).
\end{align*}
\noindent

This modification gives us the following combined Borda scores for 
all the candidates (assuming $s_{base} = P_1 \cup P_2$):
%
%
\begin{align*}
&s(v_i,P) = s_{base}(P) - m(t+5) \\
&s(a_j,P) = s_{base}(P) - m(t+6) \\
&s(p,P) = s_{base}(P) - m(t+6) \\
&s(b,P) = s_{base}(P) +mq - m(t+6) \\
&s(d,P) = s_{base}(P) + m(7t+5-q) + (mt^2) + 2m(t+6).
\end{align*}

Now all candidates in $A$ are tied with $p$ in the first round and therefore 
the tie-breaking rule must choose one to remove in each round.  In round
$4k = 0, \dots, q/3$ we must select some set of candidates $a_1, \dots, a_{q/3}$ to eliminate
(in the interleaving rounds $4k+1, 4k+2, 4k+3$, the elements $v_i$ in the set $S_j$ corresponding
to $a_j$ will drop out).  At each $4k$, $p$ will be tied with the remaining candidates in $A$ which correspond
to sets $S$ until there are no more
sets to cover (after $4q/3$ rounds).  Then $p$ will be tied with $b$ if and only if we have eliminated a cover
 and the remaining $a_j$ that were not part of the cover.  We then select $p$ to win over $b$.
There is a solution to the
X3C instance if and only if there is a selection of $q/3$ elements of $A$ that exactly cover the elements
of $V$.
\end{proof}


We move on to the Coombs rule which 
successively eliminates the candidate with 
the largest number of last place votes.

\begin{mytheorem}
		\label{l:l2}
The control by tie-breaking problem for
Coombs rule is NP-complete.
\end{mytheorem}		

\begin{proof}\emph{(Sketch)}
The result holds for both the simplified and
unsimplified Coombs rule. 
Starting from Theorem 3 in Davies et al.~\cite{DaviesNW12}
which shows that the manipulation problem for Coombs is
NP-complete we modify the profile $E$ to show hardness.
%
%
With a slight modification of the scores,
increasing the initial veto scores of $s_2$ and $d_0$ by 1 each,
we move the burden from the manipulator to the tie-breaking rule.
The profile $E$, generally, creates a voting instance where
a cover is selected and then verified through sequential eliminations
through a complex setting of initial candidate scores (see Table 1 in~\cite{DaviesNW12}).
Here, we need to show that the influence that a single manipulator
has on the outcome of the election can be
again simulated by the tie-breaking rule.

We observe that the manipulator only changes
the outcome of a round in two cases: (1) when two candidates 
are tied to select the loser during the first $4m$ rounds and 
the manipulator can \emph{only} change the outcome of $E$ at rounds $p \in \{1,5,9,\ldots,4m\}$,
where exactly two candidates are tied  \cite{DaviesNW12}.
The first case is when two candidates $a$ and $b$ are tied so that
a static tie-breaking rule should be used  to decide the loser of this round.
The manipulator ranks $a$ (or $b$) at the bottom of his preference
profile and decides which candidate is eliminated at this round regardless
of the tie-breaking rule. 

This case occurs $m$ times during 
the first $4m$ rounds in the proof in~\cite{DaviesNW12}. The manipulator can \emph{only}
change the outcome of $E$ at rounds $p \in \{1,5,9,\ldots,4m\}$, where exactly two candidates
are tied.  In this case, we can simulate the manipulator's influence using the tie-breaking rule.
(2) Where we increase the veto-score of a some candidate $a$ to tie him with a candidate $b$
so that the static tie-breaking rule eliminates $b$ at this round.  This occurs after all
$d_1,\ldots,d_n$ are eliminated (when there is a cover) allowing the elimination of $d_0$ and
after $d_0,d_1 \ldots,d_n$ are eliminated, the score of $c$ is the same as the score of $s_2$.


Increasing the initial scores of $d_0$ and $s_2$ by $1$
means $d_0$ and $c$ are tied after $d_1,\ldots,d_n$ are eliminated
and the tie-breaking rule can be used to eliminate $d_0$.  Additionally,
$c$ and $s_2$ have the same score and we can eliminate
$s_2$ before $c$ with the tie-breaking rule.
%
%
The elimination order at the $4$th stage is independent of the
manipulator's vote. Hence, $p$ wins
the election if and only if we select a cover during the
first $4m$ rounds by means of breaking ties
appropriately. Hence, control by tie-breaking for Coombs is NP-complete.
\end{proof}

The construction in the proof of Theorem~\ref{l:l2} can be used to state the following corollary for the (unsimplified) Coombs rule.

\begin{mycorollary}
The control by tie-breaking problem for the
Coombs rule is NP-complete.
\end{mycorollary}

\subsection{Cup and Copeland}

Cup and Copeland are often used in real life settings involving sports or other 
competitions where ties must be resolved on the fly.  
Under Cup all ties must be resolved
before the next round can be computed.  In Copeland, when we only select
a winner from the set of elements tied with highest Copeland score, we fall
under the result of Theorem~\ref{thm:easy}.  However, when Copeland is used
in a sports competition, often pairwise ties between candidates need to be 
resolved before the final Copeland score can be computed (i.e., NCAA football).




To determine the best tie-breaking order for Cup
we can use the algorithm from Theorem 2
in \cite{conitzer2007elections} that computes
a manipulating vote 
and use the returned
manipulation as the linear order for tie-breaking.

\begin{mytheorem}
The control by tie-breaking problem for the Cup rule is polynomial when each candidate
appears only once in $S$.
\end{mytheorem}

Notice that the above procedure always returns a linear
tie-breaking order. When each candidate
appears only once, a manipulator cannot
benefit by breaking ties with an
order that violates transitivity.  To profit
from a non-transitive order,
we would need multiple pairwise comparisons between candidates.
For example, in double elimination tournaments such as the Australian Rules 
Football League Finals Series,
candidates appear twice.
We only need a linear schedule and one candidate to appear twice to see the difference.
Consider the tournament illustrated in Figure~\ref{fig:tie}; in order to select
between $a$, $b$, and $c$ to ensure a win for $p$ we must choose a 
non-linear order where $c>b$, $b>a$ and $a>c$.


\begin{figure}
\centering
\includegraphics[scale=0.2]{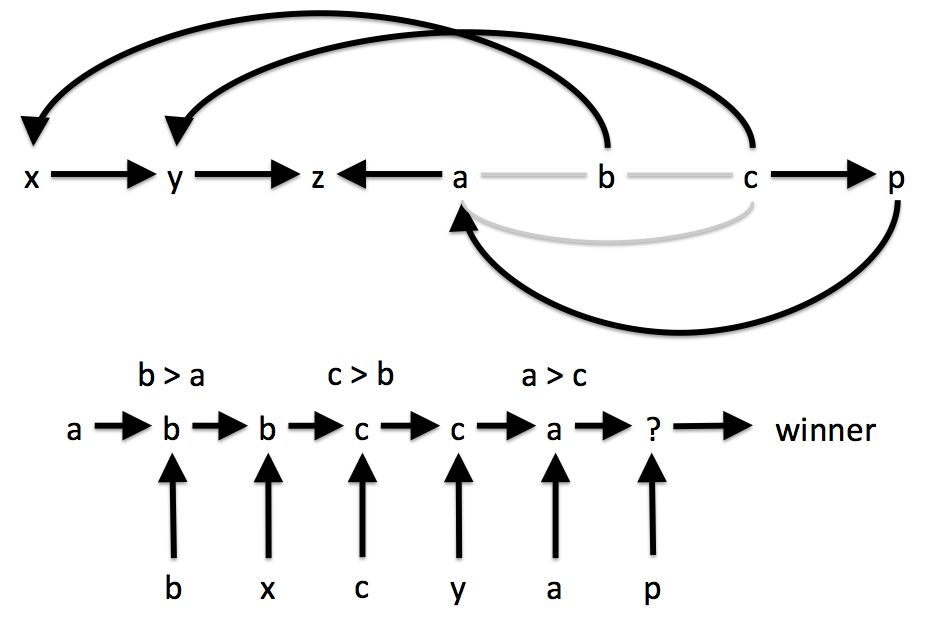}

\caption{Pairwise relation and Cup graph illustrating 
that it is possible to increase the chances for control if we allow the chair to specify a 
non-transitive tie-breaking order.}
\label{fig:tie}
\end{figure}

If we allow candidates to enter the tournament more than once, and if the tournament
can have arbitrary shape, the control by tie-breaking problem becomes hard.

\begin{mytheorem}\label{cup:3sat}
When the Cup schedule $S$ can have arbitrary shape and candidates can appear
more than once, control by tie-breaking is NP-complete.
\end{mytheorem}

\begin{proof} 
We reduce from an instance of \textsc{3SAT} where we are given
a set of clauses $K = \{ k_1, \dots, k_n\}$ and a set of literals with their negations
$L = \{l_1, \overline{l_1}, \dots, l_m, \overline{l_m}\}$ and asked to find an assignment
to every literal, either $l_1$ or $\overline{l_1}$ such that every clause in $K$ is satisfied.

Given an instance of \textsc{3SAT}, we create a Cup tournament with candidates
$C = K \cup L \cup \{p\}$.  The pairwise relation between the candidates in the cup
has $p$ defeating all elements
of $L$; each literal and its negation are tied and each literal and its negation defeat any literal with a higher number
($l_1 \sim  \overline{l_1} > l_2 \sim \overline{l_2}$); each clause is defeated by only those literals
that would satisfy it while defeating all other literals and $p$.  We construct $S$, the cup tournament,
 as follows: for each $k_i$ we pair the three literals that would satisfy $k_i$ with their negations in a sub-cup
 where each literal matches against its negation in the first round.  In the second to fourth round of the sub-cup
 $k_i$ plays the winner of each of these three literal vs. literal match-ups, sequentially.  We then compose each
 of these sub-cups, for each $k_i$ such that $p$ faces the winner of each of the sub-cups
 sequentially.  Thus every literal plays the clause that contains it and every clause plays $p$.

The manipulator must select a tie-breaking order for each pair of literals and their negations.
If there is a tie-breaking order which selects either $l_1$ or $\overline{l_1}$ for every literal
such that all clauses are satisfied, then $p$ will win the Cup.  Each $k_i$ will face
each literal or negation in $k_i$, depending on the tie-breaking rule.  Since only literals that
satisfy $k_i$ defeat it, one must be selected, otherwise $p$ will lose to $k_i$ when $p$
plays the winner of each of the sub-tournaments.  If $k_i$ is satisfied by one (or more)
of its literals, then $k_i$ will be eliminated and a literal will face (and lose to) $p$ in the latter part of the tournament.
Hence $p$ will win the tournament if and only if there is tie-breaking rule that satisfies the
\textsc{3SAT} instance.
\end{proof}

For the Copeland rule,
we know that the second-order tie-breaking
rule is NP-hard to manipulate \cite{bartholditoveytrick}.
We can also devise other tie-breaking rules to add
to Copeland to make manipulation NP-hard \cite{oehaamas11}.
On the other hand, the regular Copeland rule is vulnerable
to control by breaking ties.  Here we consider the variant of Copeland often
used in conjunction with sports tournaments
where we must resolve ties in the pairwise graph \emph{before} we resolve
any ties related to the total Copeland score (e.g., Olympic Ice Hockey).  
This problem is closely
related to the possible winner problem when there are partially specified
preference profiles \cite{xia2011determining}.
%
%
If we require that the tie-breaking rule be transitive then
we can use the algorithm from Bartholdi, Tovey and Trick
\cite{bartholditoveytrick} which provides a score minimizing
the maximum score of all other candidates to set the tie-break
order.


Allowing the tie-breaking rule to be non-transitive increases the potential for
control of the tie-breaking rule under Copeland. Consider the
election with the following 6 votes:
$(d,f,g,p,a,b,c),$
$(d,f,g,p,c,b,a),$
$ 2 \times (p,a,b,c,d,f,g),$
and $ 2 \times (c,b,a,d,f,g,p)$. Suppose
we want $p$ to win.  The Copeland
score of the candidates (denoted s(c)) are $s(p)=s(a)=s(b)=s(c)=s(d)=s(3), s(f)=2, s(g)=0$.  We need to submit a tie-breaking
order that will resolve the pairwise ties between $a,b$ and $c$.  There is no transitive order that we can submit
to resolve these such that $p$ wins.  However, we can submit pairwise preferences that will maintain
the cycle and allow $p$ to win with an additional point over $g$.

We note that such tie-breaking is closely to the problem
of manipulating a Copeland election with \textit{irrational} voters
\cite{faliszewski2009llull} which
is polynomial time computable \cite{faliszewski2010manipulation}.  If we allow the
chair to specify the result of each pair-wise
tie separately (and thus to break ties
non-transitively), Copeland remains vulnerable
to control by breaking ties.  This problem 
is closely related to the problem of finding possible 
winners in a tournament that is partially specified \cite{aziz2012possible}.
We can use the algorithm presented in Theorem 5.2 from Faliszewski et al.~\cite{faliszewski2010manipulation}
which allows us to find an assignment to the pairwise relationship
between all the non-$p$ candidates to minimize their scores.

\begin{mycorollary}\label{cor:copeland}
The control by tie-breaking
problem for Copeland is polynomial,
even when tie-breaking is specified in terms
of a non-transitive ordering on pairwise
contests.
\end{mycorollary}




\section{COMBINING VOTING RULES}

We have seen many rules are vulnerable to control by breaking ties 
when there is only one opportunity to break ties.  Conversely, we 
have seen that some rules
are resistant when many tie-breaks are required.
This leaves open the question of what happens
when only a small, fixed number of tie-breaks are required.
Interestingly, we show that in these cases, rules can be 
resistant to control by tie-breaking.

%

\begin{mytheorem}
There exists a two stage voting rule based on veto and plurality
where the control
by tie-breaking problem is NP-complete.
\label{thm:vetoplurality}
\end{mytheorem}
\begin{proof}
We consider the
rule that first eliminates
half the candidates using the
veto rule, then elects the
plurality winner.
Clearly, the control problem is in NP.
We select the subset of candidates through tie-breaking
for the runoff. To show NP-hardness,
we adapt the reduction from \textsc{X3C} used in the proof of
Theorem 3 in \cite{bttmcm92} that demonstrates
control by elimination of candidates
for plurality is NP-hard.

Given two sets $V = \{v_1, \dots, v_m\}$ and $S = \{S_1, \dots, S_n\}$,
this reduction uses $n+4m/3+2$ candidates
where $m$ is the size of the set being
covered, $n$ is the number of 3-element
sets from which the cover is built:
$p$ is the preferred  candidate,
$w$ is the current winner,
$s_i$, $i = 1,\ldots,n$ are candidates that correspond
to sets $S_i$, $v_j$, $j=1,\ldots,m$ represent the elements of $V$;
$f_k$, $k = 1,\ldots, m/3$ are additional candidates.
We double the number of candidates to $2n+8m/3+4$
with $n+4m/3+2$ additional dummy candidates $d_i$ that
occur in the same fixed order in every vote.
The first $n+m+1$ candidates appear at the front
of the votes, whilst the last $m/3+1$ appear
at the end. The unlisted candidates are ranked in an arbitrary order between 
the candidates $f_{m/3}$ and $d_{m+n+2}$.


\begin{itemize}
  \item for $i=1,\ldots,n$: 1 vote $(d_1 \succ \ldots \succ d_{n+m+1} \succ s_i \succ c \succ f_1 \succ \ldots \succ f_{m/3} \succ \ldots  \succ d_{n+m+2} \succ \ldots \succ d_{n+4m/3+2})$
  \item for $i=1,\ldots,n$ and $v_i^j \in S_i, j=1,2,3$: 1 vote $(d_1 \succ \ldots \succ d_{n+m+1} \succ s_i \succ  v_i^j \succ f_1 \succ \ldots \succ f_{m/3} \succ \ldots  \succ d_{n+m+2} \succ \ldots \succ d_{n+4m/3+2})$.
\item $m/3 - 1$  votes $(d_1 \succ \ldots \succ d_{n+m+1} \succ w \succ f_1 \succ \ldots \succ f_{m/3} \succ \ldots  \succ d_{n+m+2} \succ \ldots \succ d_{n+4m/3+2})$
\item for  $j=1,\ldots,m$: $m/3 - 2$ votes  $(d_1 \succ \ldots \succ d_{n+m+1} \succ v_j \succ f_1 \succ \ldots \succ f_{m/3} \succ \ldots  \succ d_{n+m+2} \succ \ldots \succ d_{n+4m/3+2})$.
\end{itemize}


With our two stage rule, one
of the dummy candidates, $d_{n+4m/3+2}$, has all the vetoes so will
be eliminated. The chair therefore has to tie-break and select, sequentially, $n+4m/3+2$ of the 
remaining $2n+8m/3+3$ candidates.  We start with $W = C \setminus \{d_{n+4m/3+2}\}$
and apply the tie-breaking rule sequentially $m/2$ times. At each step we select some candidate,
$T(W,P) = c_i$, then re-apply
the tie-breaking rule on the set $T(W\setminus\{c_i\},P)$, continuing until we select a set of
the correct size for the second round.  
To ensure that the distinguished
candidate is the plurality winner, the chair's tie-breaking must
eliminate all $n+m+1$ dummy candidates at the front
of the vote, plus $m/3$ of the candidates
from the original election corresponding to the
cover. Hence, the \textsc{X3C} problem has a solution if and only
if the chair can tie-break to
ensure the distinguished candidate wins.
\end{proof}

Conitzer and Sandholm \cite{csijcai03} give a general construction
that builds a two-stage voting rule that often makes
it intractable to compute a manipulating vote.
This construction runs one round of the Cup rule,
eliminating half of the candidates, and then
applies the original base rule to
the candidates that remain.
For the base rule $X$, we denote this
as $Cup_1+X$. The control by tie-breaking problem is also
typically intractable for
such two-stage voting rules.  Adapting Theorem 2 in Conitzer and Sandholm \cite{csijcai03}
we can make the following statement.

\begin{mytheorem}
The control
by tie-breaking problem for
$Cup_1+Plurality$,
$Cup_1+Borda$, and
$Cup_1+Maximin$
are NP-complete.
\end{mytheorem}

\begin{proof}
Consider the reduction from
SAT used in Theorem 2 in Conitzer and Sandholm \cite{csijcai03}
showing that it is NP-hard to
construct a single vote to ensure
a distinguished candidate wins $Cup_1+Plurality$.
This reduction sets up a profile
in which the candidates $c_{+v}$ and $c_{-v}$ corresponding
to a literal and its negation which are paired in the first
round of a Cup and are tied.  The construction of the Cup
is similar to the one described in Theorem~\ref{cup:3sat}.
There is a vote that breaks these
ties so that the distinguished candidate wins
if and only if the SAT instance is satisfiable.
Let us consider just the original profile, without
the single manipulating vote. Now, the chair can break these
ties so that the distinguished candidate wins
if and only the SAT instance is satisfiable. The other
proofs are similar and are adapted
from reductions in
\cite{csijcai03}.
\end{proof}

Elkind and Lipmaa
\cite{elisaac05} generalize this
construction to run a number of rounds, $k$,
of some rule before
calling a second rule; making computing a manipulating vote
NP-hard in many cases. Control by tie-breaking for such hybrids
is often NP-hard as tie-breaking can simulate the
manipulating vote used in the proofs in Elkind and Lipmaa \cite{elisaac05}.
For example, control by tie-breaking
for $\Hyb(\STV_k,Y)$ and $\Hyb(Y,\STV_k)$ is
NP-hard, where
$Y$ is one of: plurality, Borda, maximin or Cup.
%
%
Interestingly,
$\Hyb(plurality_k, plurality)$ is vulnerable
to manipulation~\cite{elisaac05} and manipulation is polynomial 
if $k$ is bounded; this result carries to our problem.
However, this hybrid is resistant to control by tie-breaking
for unbounded $k$.

\begin{mytheorem}
The control by tie-breaking problem for \\ $\Hyb(Plurality_k, Plurality)$
is polynomial if $k$ or $m-k$ is bounded.
\end{mytheorem}

\begin{proof}
If $k$ is bounded, we can
try all $O(m^k)$ possible
tie-breaking decisions about
candidates to eliminate.
Similarly, if $m-k$ is bounded,
we can try all $O(m^{m-k})$ possible
tie-breaking decisions about candidates
to survive.
\end{proof}

\begin{mytheorem}
The control by tie-breaking problem for $\Hyb(Plurality_k, Plurality)$ if $k$ is unbounded is NP-complete and 
polynomial if $k$ or $m-k$ is bounded.
\end{mytheorem}
\begin{proof}
If $k$ is bounded, we can
try all $O(m^k)$ possible
tie-breaking decisions about
candidates to eliminate or all $O(m^{m-k})$ possible
tie-breaking decisions about candidates
to survive if $m-k$ is bounded.

When $k$ is unbounded our construction is similar to the construction in the proof of Theorem 3 in Elkind and Lipmaa \cite{elisaac05}.
We reduce from an instance of the \textsc{X3C} problem where each item occurs
in at most 3 subsets. We are given a set of items $V = \{v_1,\ldots,v_m\}$ with $|V| = m$ and
subsets $S_1, S_2, \ldots,  S_n \subset V$ with $|S_i| = 3$ for $i=1,\ldots,n$.
The question is whether there exists an index set $I$ with $|I| = m/3$ and
$\bigcup_{i\in I} S_i = S$. We build an election with $n+m+2$ candidates: $C = V \cup S \cup \{p, d\}$.
We have $m$ candidates $V = \{v_1,\ldots,v_m\}$
that encode items, $n$ candidates  $S = \{s_1,\ldots,s_n\}$ that encode sets,
a  dummy candidate $d$ and the preferred
candidate $p$.  Let $T$  be a constant $\geq3nm$.

We introduce the following two sets of votes, $P= P_1 \cup P_2$.
We denote $S^i = \{s_j| v_i \in S_j\}$.
The first set $P_1$ contains the following votes:
\begin{itemize}
  \item for each $v_i$, $i=1,\ldots,m$: $T$ votes $(p \succ C \setminus \{p\})$
  \item for each $v_i$, $i=1,\ldots,m$:  $T-2$ votes
$(v_i \succ C \setminus \{v_i\})$
  \item for each $S^i$, $i=1,\ldots,m$: $3$ votes $(S^i \succ v_i \succ C \setminus S^i)$
  \item  $4$ votes $(d \succ C \setminus \{d\})$
\end{itemize}

To build $P_2$, let $n_j$ be the number of first places occupied by $s_j$ in $P_1$, 
thus $n_j \leq 3$. We introduce 
$3 - n_j$ votes $(s_j \succ d\succ C \setminus \{s_j, d\})$.
The rest of the votes are irrelevant.
Thus, the initial plurality scores of the candidates are:
$score(p) = T$, $score(v_i) = T-2$, $i=1,\ldots,m$, $score(s_j) = 3$, $j=1,\ldots,n$
and $score(d) = 4$.  We set $k = n-m/3$.

During the first $n-m/3$ rounds, $n - m/3$ candidates from $S$ are eliminated
and the tie-breaking rule decides which $n - m/3$ out of $m$ candidates
to eliminate as all $n$ candidates in $S$ are tied. If the remaining $m/3$ candidates
in $S$ are not a cover, then an uncovered item $v_i$ gets
3 points resulting in a plurality score of $T+1$. Hence, $p$ loses.
Therefore, tie-breaking must ensure that the remaining candidates from $S$
form a cover. Finally, if a valid cover is selected,
the maximal plurality score of $d$ after
$k$ rounds is $4+3n$, the maximal plurality score of any surviving $s_j$
is $9$, the maximal plurality score of $v_i$ is $T-2$ and the score of $p$
is $T$. Hence, $p$ wins iff there is a cover.
\end{proof}

\section{CONCLUSION}
We have studied the computational complexity of
the control by tie-breaking problem.  
This problem is equivalent to the problem of deciding 
the winner of an election under
PUT.
When the chair is only asked
to choose between the co-winners,
the problem is trivially polynomial. However, in
multi-round elections, where the chair
may have to break multiple ties, we proved
that this control problem can be NP-complete, and 
the form of the tie-breaking function can increase the opportunities
for control.  Table 1 provides a summary of these results.

\begin{table}
\centering
\begin{tabular}{c|c}
\toprule

P & NP-complete \\ \midrule
scoring rules, Cup,  		 &    STV, Baldwin \\
Nanson, Copeland, maximin     		&   ranked pairs,  \\
Bucklin, fallback, Schulze		&  Coombs \\
Kemeny-Young ($m \leq 3$) & Kemeny-Young ($m \geq 3$) \\
 \bottomrule
\end{tabular}
\caption{Complexity of control by tie-breaking.
The result for ranked pairs is due to \protect\cite{BrillF12} and the result
for STV is due to \cite{conitzer2009preference}.}
\end{table}

Interestingly, with a two-stage voting rule, even though 
the chair might only be asked to break ties at most twice, 
control by tie-braking can be NP-hard.
Of course, many of our results are worst-case and may not reflect
the difficulty of manipulation in practice. A number of
recent theoretical and empirical results
suggest that manipulation can often be computationally
easy on average
(e.g. \cite{mattei2012empirical,prjair07,xcec08,xcec08b,fknfocs09,wijcai09,wecai10}).
We intend to explore the hardness of control by tie-breaking using 
data from PrefLib \cite{preflib} and other sources.

%
%
%
%

\smallskip
\noindent
\textbf{Acknowledgements.}~~NICTA is funded by the Australian Government through the 
Department of Communications and the Australian Research Council 
through the ICT Centre of Excellence Program.

\bibliographystyle{ecai2014}
\bibliography{control-tb}
\end{document}